\newif\ifpaper

\paperfalse 

\ifpaper
	\documentclass[twoside,11pt]{article}
	\usepackage{jair, rawfonts}
	\usepackage[natbibapa]{apacite}
\else
	\documentclass{article}
	\usepackage[final,nonatbib]{nips_2016_tr}
\fi

\usepackage[utf8]{inputenc} 
\usepackage{hyperref}       
\usepackage{url}            
\usepackage{booktabs}       
\usepackage{amsfonts}       
\usepackage{nicefrac}       
\usepackage{microtype}      
\usepackage{mathbbol}

\usepackage{amsthm}
\usepackage{amsmath}
\usepackage{mathtools}
\usepackage{graphicx}
\usepackage{multicol,array}
\usepackage{comment}
\usepackage{subcaption}
\usepackage{acronym}
\usepackage{hyperref}
\usepackage{tabulary}
\usepackage{array}
\usepackage{algpseudocode}
\usepackage{algorithmicx}
\usepackage{standalone}
\usepackage{mathrsfs}
\usepackage{tikz}
\usepackage{bbm}
\usepackage{framed}
\usepackage{rotating}
\usepackage{amssymb}

\acrodef{HMM}{hidden Markov model}
\acrodef{KFM}{Kalman filter model}
\acrodef{CPD}{conditional probability distribution}
\acrodef{DBN}{dynamic Bayesian network}
\acrodef{DAG}{directed acyclic graph}
\acrodef{BN}{Bayesian network}
\acrodef{MDL}{minimum description length}
\acrodef{PDF}{probability density function}
\acrodef{BIC}{Bayesian information criterion}
\acrodef{AIC}{Akaike information criterion}
\acrodef{GDS}{graph dynamical system}
\acrodef{CPT}{conditional probability table}
\acrodef{EM}{expectation-maximisation}
\acrodef{RPDAG}{restricted partially directed acyclic graph}
\acrodef{ODE}{ordinary differential equation}

\DeclareMathOperator*{\argmax}{arg\,max}
\DeclareMathOperator*{\argmin}{arg\,min}

\newtheorem{theorem}{Theorem}
\newtheorem{corollary}{Corollary}[theorem]
\newtheorem{lemma}[theorem]{Lemma}
\newtheorem{definition}{Definition}

\theoremstyle{remark}

\newcommand{\oc}[1] {{\color{magenta}\textbf{(OC)} #1}}

\ifpaper
\else
	\newcommand{\citep}[1]{\cite{#1}}
\fi

\ifpaper 
	\title{Learning the Structure of\\Distributed Dynamical Systems}
\else
	\title{Inferring Coupling of Distributed Dynamical Systems via Transfer Entropy}
\fi

\ifpaper
	\ShortHeadings{Learning Nonlinear Dynamical Networks}
	{Cliff, Prokopenko, \& Fitch}
	\firstpageno{1}
	
	\author{\name Oliver M. Cliff \email o.cliff@acfr.usyd.edu.au \\
	       \addr Australian Centre for Field Robotics,\\
	       The University of Sydney,\\
	       Sydney, NSW, 2006, Australia
	       \AND
		   \name Mikhail Prokopenko \email mikhail.prokopenko@sydney.edu.au \\
	       \addr Complex Systems Research Group,\\
	       The University of Sydney,\\
	       Sydney, NSW, 2006, Australia
	       \AND
	       \name Robert Fitch \email r.fitch@acfr.usyd.edu.au\\
	       \addr Centre for Autonomous Systems,\\
	       University of Technology Sydney,\\
	       Sydney, NSW, 2007, Australia}
\else
	\author{Oliver M.~Cliff$^*$ \And Mikhail Prokopenko$^\dagger$\And Robert Fitch$^{*,\ddagger}$\\
			  \AND \\[-5mm]
			  $^*$Australian Centre for Field Robotics,\\
			  The University of Sydney, Australia,\\
			  \texttt{o.cliff@acfr.usyd.edu.au}
			  \And \\[-5mm]
			  $^\dagger$Complex Systems Research Group,\\
			  The University of Sydney, Australia,\\
			  \texttt{mikhail.prokopenko@sydney.edu.au}
			  \AND \\[-5mm]
			  $^\ddagger$Centre for Autonomous Systems,\\
			  University of Technology Sydney, Australia,\\
			  \texttt{robert.fitch@uts.edu.au}}
\fi

\begin{document}

\maketitle

\begin{abstract}
	In this work, we are interested in structure learning for a set of spatially distributed dynamical systems, where individual subsystems are coupled via latent variables and observed through a filter. We represent this model as a directed acyclic graph (DAG) that characterises the unidirectional coupling between subsystems. Standard approaches to structure learning are not applicable in this framework due to the hidden variables, however we can exploit the properties of certain dynamical systems to formulate exact methods based on state space reconstruction. We approach the problem by using reconstruction theorems to analytically derive a tractable expression for the KL-divergence of a candidate DAG from the observed dataset. We show this measure can be decomposed as a function of two information-theoretic measures, transfer entropy and stochastic interaction. We then present two mathematically robust scoring functions based on transfer entropy and statistical independence tests. These results support the previously held conjecture that transfer entropy can be used to infer effective connectivity in complex networks.
\end{abstract}

\section{Introduction}

Complex networks are capable of modelling a wide array of important phenomena in both natural and artificial environments~\citep{boccaletti06a}. This work focuses on a particular complex network which comprises spatially distributed dynamical systems. We represent this network as a type of probabilistic graphical model termed a synchronous \ac{GDS}~\citep{mortveit07a,cliff16b}. We are interested in \emph{structure learning} for synchronous \acp{GDS} whose structure is given by a \ac{DAG}, that is, the problem of inferring directed relationships between hidden variables from an observed dataset. We propose a solution based on the concept of \emph{transfer entropy}, which is a measure that detects the directed information-theoretic dependency between random processes~\cite{schreiber00a}. Specifically, we prove, under certain technical assumptions of the system, that the maximum transfer entropy graph is the optimal information-theoretic model. We then employ this result in developing a number of mathematically robust scoring functions for the structure learning problem.

This structure learning problem has applications in a wide variety of areas due to its usefulness for performing efficient inference in discrete-time dynamical systems, in addition to understanding the system's complex structure. Dynamical systems are characterised by a map that describes their evolution over time and a read-out function through which we observe the latent state. These systems are ubiquitous in the literature due to their ability to model many real-world phenomena. Our research focuses on the more general case of a multivariate system, where a set of these subsystems are spatially distributed and unidirectionally coupled to one another. The problem of inferring this coupling is an important multidisciplinary study in fields such as multi-agent systems~\cite{cliff16a,umenberger16a}, ecology~\cite{sugihara12a}, neuroscience~\cite{vicente11a,schumacher15a}, and various others studying artificial and biological systems~\cite{boccaletti06a}.


A main challenge in structure learning for \acp{DAG} is the case where variables are unobserved. Exact methods are known for fully observable systems~\citep{daly11a}, however, these are not applicable because the state variables in dynamical systems are latent. Our goal in this paper is to exploit results from differential topology in inferring hidden coupling. Specifically, the main focus of this paper is to analytically derive a measure for comparing a candidate graph to the underlying graph that generated a measured dataset. Such a measure can then be used to solve the two subproblems that comprise structure learning, \emph{evaluation} and \emph{identification}~\citep{chickering02a}, and hence find the optimal model that explains the data.

Our approach in deriving a measure that can be used to solve the evaluation problem can be described in terms of model selection. It is desirable to select the \emph{simplest} model that incorporates all statistical knowledge. This concept is commonly expressed via information theory, where an established technique is to evaluate the encoding length of the data, given the model~\citep{akaike73a,schwarz78a,rissanen78a}. The simplest model should aim to minimise code length~\citep{lam94a}, and therefore we can simplify our problem to that of minimising KL divergence for the synchronous \ac{GDS}. Using this measure, we find a factorised distribution (given by the graph structure) that is closest to the joint distribution. We first analytically derive an expression for this divergence, and build on this result to present principled methods for evaluating candidate graphs based on a dataset.

The main result of this paper is a tractable expression of the KL divergence for synchronous \acp{GDS}. We show that this measure can be decomposed as the difference between two well-known information-theoretic measures, stochastic interaction~\cite{ay03a} and collective transfer entropy~\cite{lizier10b}. We establish this result by first representing discrete-time multivariate dynamical systems as \acp{DBN}~\citep{murphy02a}. In this form, both the joint and factorised distributions cannot be directly computed due to the hidden system state. Thus, we draw on methods from differential topology for state space reconstruction to reformulate the KL divergence in terms of computable distributions. Using this expression, we develop two scoring functions based on transfer entropy and independence tests.

The significance of this result is that it provides a rigorous foundation for model selection in synchronous \acp{GDS}. As we will show, maximising transfer entropy minimises the KL divergence in these systems. Interestingly, transfer entropy has already been used in practice for inferring effective networks~\citep{lizier12c} with encouraging empirical results. Our work lends mathematical justification to this approach under the given circumstances, and contributes to new potential applications of structure learning in robotics, complex systems analysis, and other areas.

\section{Related Work}

A complex network is a graph with non-trivial topological features that gives rise to emergent behaviour not typically seen in more traditional fields of graph theory~\citep{boccaletti06a}. This concept was popularised by the seminal work of \ifpaper\citep{watts98a}\else Watts and Strogatz~\citep{watts98a} \fi on small-world networks and \ifpaper\citep{barabasi99a}\else Barab\'{a}si and Albert~\cite{barabasi99a} \fi on scale-free networks. Since then, most of the complex network literature focuses on characterising the structure and dynamics of known biological, physical, and artificial networks~\citep{boccaletti06a}. We instead focus on the structure learning problem, a general paradigm in machine learning where the goal is to infer relationships between the variables within a system~\citep{koller09a}. In particular, we are interested in systems whereby the subsystems are unidirectionally coupled to one another. Besides complex networks, these types of systems have been introduced under a variety of more specific terms, such as spatially distributed dynamical systems~\cite{kantz04a,schumacher15a} and master-slave configurations~\cite{kocarev96a}. The defining feature of these networks is that the dynamics of each subsystem are given by a set of either discrete-time maps or first-order \acp{ODE}. In this paper we use the discrete-time formulation, where a map is obtained numerically by integrating \acp{ODE} or recording observations at discrete-time intervals~\cite{kantz04a}.

An important precursor to network reconstruction is inferring causality and coupling strength between complex nonlinear systems. In this work, we restrict our attention to methods that determine conditional independence (coupling) rather than causality; algorithms of this kind are applicable when the experimenter can not intervene with the dataset~\cite{pearl14a}. In early work, \ifpaper \cite{kolmogorov59a} \else Kolmogorov~\cite{kolmogorov59a} \fi introduced the concept of classification of dynamical systems by information rates, leading to a generalisation of entropy of an information source~\citep{sinai59a}. Following this, \ifpaper \cite{granger69a} \else Granger~\cite{granger69a} \fi proposed \emph{Granger causality} for quantifying the predictability of one variable from another. Although this measure has been used numerous times in identifying coupling, a limiting assumption of Granger causality is the key requirement of linearity, implying subsystems can be understood as individual parts~\citep{sugihara12a}. \ifpaper \cite{schreiber00a} \else Schreiber~\cite{schreiber00a} \fi extended the ideas of Granger and introduced \emph{transfer entropy} using the concept of finite-order Markov processes to quantify the information transfer between coupled nonlinear systems (although this idea was expressed earlier by \ifpaper \citep{marko73a} \else Marko~\citep{marko73a} \fi as an information-theoretic interpretation of predictability). Interestingly, it was recently shown that the two approaches are linked in linearly-coupled Gaussian systems (e.g., Kalman models~\citep{murphy02a}), where transfer entropy and Granger causality are equivalent~\citep{barnett09a}. However, there are clear distinctions between the concepts of information transfer and causal effect (see, e.g., the analysis in \ifpaper\cite{lizier10c} \else Lizier and Prokopenko~\cite{lizier10c} \fi).

Recently, a number of measures have been proposed to infer coupling between distributed dynamical systems based on state space reconstruction theorems~\cite{sugihara12a,schumacher15a,cliff16b}. \ifpaper \cite{sugihara12a} \else Sugihara et al.~\cite{sugihara12a} \fi assumed Granger's definition of causality as a quantification of predictability and proposed a method labelled convergent cross-mapping (CCM). CCM involves collecting a history of observed data from one subsystem and uses this to predict the outcome of another subsystem. This history is the delay reconstruction map described by Takens' Delay Embedding Theorem~\citep{takens81a}. Similarly, \ifpaper \cite{schumacher15a} \else Schumacher et al.~\cite{schumacher15a} \fi used the Bundle Delay Embedding Theorem~\cite{stark97a} infer causality and perform inference via Gaussian processes. Although the algorithms presented in these papers can infer driving subsystems in a spatially distributed dynamical system, the results obtained differ from ours as inference is not considered for an entire network structure, nor is a formal derivation presented. Finally, we recently presented similar work on deriving an information criterion for learning the structure of distributed dynamical systems~\cite{cliff16b}. However, the criterion we proposed was both only asymptotically optimal and required parametric modelling of the probability distributions. In this paper we extend this framework by proposing two scoring functions: one that is comparable to the information criterion presented in~\cite{cliff16b} in that it is applicable for discrete and linearly-coupled Gaussian variables; and another that allows for non-parametric density estimation techniques and thus make no assumptions about the underlying distributions.
 
A major contribution of this paper is a formal proof that maximising collective transfer entropy in a network reveals the information-theoretically optimal structure. A related line of inquiry is recovering \emph{effective networks}: networks that reveal the ``effective structure" of an observed system~\citep{sporns04a,park13a}. Using transfer entropy to infer effective networks has become a popular transdisciplinary analysis technique, e.g., in computational neuroscience~\citep{vicente11a,lizier12c}; multi-agent systems~\citep{cliff13a,cliff16a}; financial markets~\citep{sandoval14a}; supply-chain networks~\citep{rodewald15a}; and gene regulatory networks~\citep{damiani11a}. However, there is a dearth of work that provide formal derivations for the use of this measure in inferring effective structure. Most of the results build on Schreiber's work~\cite{schreiber00a} and assume the system to be composed of finite-order Markov chains; we extend this notion by showing that transfer entropy can also reveal the effective structure of distributed dynamical systems. In prior work~\cite{cliff16b}, we have connected the log-likelihood ratio of a distributed dynamical system and transfer entropy. However, in this paper we arrive at this result directly by considering the minimal code length of the graph structure and present scores based on this result.

In order to evaluate the quality of a network structure, we adopt the framework of \acp{DBN}~\citep{murphy02a}. In \ac{BN} structure learning literature, there is an already mature research topic called the \emph{evaluation problem}, which is aimed at deriving a measure that can be used to score candidate graphs, given a dataset~\citep{chickering02a}. A number of mathematically sound techniques exist for the evaluation problem in a fully observed \ac{BN}~\citep{bouckaert94a,heckerman95a,heckerman95b,buntine91a}, most of which can be readily extended to the \ac{DBN} case~\citep{friedman98a}. With hidden variables, however, these guarantees do not hold and authors will often rely on heuristics. \ifpaper \cite{russell95a} and \cite{binder97a} \else Russell et al.~\cite{russell95a} and Binder et al.~\cite{binder97a} \fi use gradient descent to find parameters with possible hidden variables, and then extended their work to continuous nodes and \acp{DBN}. \ifpaper \cite{kwoh96a} \else Kwoh and Gillies~\cite{kwoh96a} \fi use an \emph{ad hoc} method to invent hidden nodes for unexplained data. \ifpaper \cite{bishop98a} \else Bishop et al.~\cite{bishop98a} \fi focused on solutions for cases specific to a sigmoid network with mixtures. Although most methods for structure learning are aimed at finding a local maxima, \ifpaper \cite{chickering97a} \else Chickering et al.~\cite{chickering97a} \fi propose using the decomposability of the functions for efficient Monte Carlo methods that avoid this caveat. In general, the above measures are derived for general \acp{BN} without any assumptions on the structure, and give only heuristic solutions. Our approach is derived specifically for multivariate dynamical systems, and we are thus afforded simplifying assumptions that allow us to develop a mathematically rigorous solution. Interestingly, the analogous concept of maximising mutual information has been previously derived as a measure to recover fully observed \acp{BN}~\citep{lam94a,bouckaert94a,campos06a} and \acp{DBN}~\citep{vinh11b}.

\section{Background} \label{sec:background}

\subsection{Notation}

In this work we consider a collection of stationary stochastic temporal processes $\boldsymbol{Z}$. Each process $Z^i$ comprises a sequence of random variables $(Z^i_{1},Z^i_2,\ldots,Z^i_N)$ with realisation $(z^i_1,z^i_2,\ldots,z^i_N)$ for countable time indices $n\in\mathbb{N}$. Given these processes, we can compute probability distributions of each variable by counting relative frequencies or by density estimation techniques~\citep{kozachenko87a,kraskov04a,victor02a}.\footnote{To simplify notation, the variables in this work appear as discrete random variables. There is no restriction on these being continuous variables; we can simply replace all sums with an integral. Obviously, this would require different density estimators as referenced here.} We use bold to denote the set of all variables, e.g., $\boldsymbol{z}_n=\langle z_n^1, z_n^2, \ldots, z_n^M \rangle$ is the collection of $M$ realisations at index $n$. Further, unless otherwise stated, $X^i_n$ is a latent (hidden) variable, $Y^i_n$ is an observed variable, and $Z^i_n$ is an arbitrary variable; thus, $\boldsymbol{Z}_n=\{\boldsymbol{X}_n,\boldsymbol{Y}_n\}$ is the set of all hidden and observed variables at temporal index $n$. Given a graphical model $G$, the $p^i$ parents of variable $Z^i_{n+1}$ is given by the parent set $\Pi_G(Z^i_{n+1}) = \langle Z^{ij}_n \rangle_j = \langle Z^{i1}_n, Z^{i2}_n,\ldots, Z^{ip^i}_n \rangle$. Finally, let the superscript $z^{i,(k)}_n = \langle z^{i}_n, z^{i}_{n-1}, \ldots, z^{i}_{n-k+1} \rangle$ denote the vector of $k$ previous values taken by variable $Z^i_n$.

\subsection{Learning Nonlinear Dynamical Networks}

We are interested in modelling discrete-time multivariate dynamical systems, where the state is a vector of real numbers given by a point $\boldsymbol{x}_n$ lying on a compact $d$-dimensional manifold $\mathcal{M}$. A map $f:\mathcal{M} \to \mathcal{M}$ describes the temporal evolution of the state at any given time, such that the state at the next time index $\boldsymbol{x}_{n+1} = f( \boldsymbol{x}_n )$. Furthermore, in many practical scenarios, we do not have access to $\boldsymbol{x}_n$ directly, and can instead observe it through a \emph{measurement function} $\psi : \mathcal{M} \to \mathbb{R}^M$ that yields a scalar representation $\boldsymbol{y}_n = \psi( \boldsymbol{x}_n )$ of the latent state~\citep{stark97a,kantz04a}. We assume the multivariate system can be factorised and modelled as a \ac{DAG} with spatially distributed dynamical subsystems, termed a synchronous \ac{GDS}. This definition is restated from~\cite{cliff16b} as follows.
\begin{definition}[Synchronous graph dynamical system (GDS)] \label{def:graph-dynamical-system}
	A synchronous \ac{GDS} is a tuple $(G,\boldsymbol{x}_n,\boldsymbol{y}_n,\{ f^{i} \}, \{ \psi^i \})$ that consists of:
	\begin{itemize}
		\item a finite, directed graph $G=(\mathcal{V},\mathcal{E})$ with edge-set $\mathcal{E} = \{E^i\}$ and $M$ vertices comprising the vertex set $\mathcal{V}=\{V^i\}$;
		\item a multivariate state $\boldsymbol{x}_n=\langle x^i_n\rangle$, composed of states for each vertex $V^i$ confined to a $d^i$-dimensional manifold $x^i_n \in \mathcal{M}^{i}$;
		\item an $M$-variate observation $\boldsymbol{y}_n=\langle y^i_n\rangle$, composed of scalar observations for each vertex $y^i_n \in \mathbb{R}$;
		\item a set of local maps $\{ f^i \}$ of the form $f^i:\mathcal{M} \to \mathcal{M}^i$, which update synchronously and induce a global map $f:\mathcal{M} \to \mathcal{M}$; and
		\item a set of local observation functions $\{ \psi^1,\psi^2,\ldots,\psi^M \}$ of the form $\psi^i:\mathcal{M}^i\to\mathbb{R}$.
	\end{itemize} 
\end{definition}
The global dynamics and observations can therefore be described by the set of local functions~\cite{cliff16b}:
\begin{gather}
	x^i_{n+1}=f^i(x^i_n, \langle x^{ij}_n \rangle_j) + \upsilon_{f^i}, \label{eq:dynamics} \\
	y^i_{n+1} = \psi^i( x^i_{n+1} ) + \upsilon_{\psi^i}, \label{eq:observation}
\end{gather}
where $\upsilon_{f^i}$ and $\upsilon_{\psi^i}$ are additive noise terms. The subsystem dynamics~\eqref{eq:dynamics} are a function of the subsystem state $x^i_n$ and the subsystem parents' state $\langle x^{ij}_n \rangle_j$ at the previous time index, i.e., $f^i:(\mathcal{M}^{i} \times_j \mathcal{M}^{ij}) \to \mathcal{M}^i$. However, the observation $y_{n+1}^i$ is a function of the subsystem state alone, i.e., $\psi^i:\mathcal{M}^i \to \mathbb{R}$. We assume the maps $\{ f^i \}$ and $\{\psi^i\}$, as well as the graph $G$, are time-invariant.

The discrete-time mapping for the dynamics~\eqref{eq:dynamics} and measurement function~\eqref{eq:observation} can be modelled as a \ac{DBN} in order to facilitate structure learning of the graph~\cite{cliff16b}. \acp{DBN} are a probabilistic graphical model that represent probability distributions over trajectories of random variables $(\boldsymbol{Z}_1,\boldsymbol{Z}_2,\ldots)$ by a prior \ac{BN} and a \emph{two-time-slice \ac{BN} (2TBN)}~\citep{friedman98a}. To model the maps, however, we need only to consider the 2TBN $B = (G,\Theta)$, which models a first-order Markov process $p_{B}( \boldsymbol{z}_{n+1} \mid \boldsymbol{z}_{n} )$ graphically and consists of: a \ac{DAG} $G$ and a set of \ac{CPD} parameters $\Theta$.~\citep{friedman98a}. Given a set of stochastic processes $( \boldsymbol{Z}_1, \boldsymbol{Z}_2, \ldots, \boldsymbol{Z}_N )$, the realisation of which constitutes a dataset $D=(\boldsymbol{z}_1,\boldsymbol{z}_2,\ldots,\boldsymbol{z}_N)$, the 2TBN distribution is given by $p_{B}( \boldsymbol{z}_{n+1} \mid \boldsymbol{z}_{n} ) = \prod_{i} p_{B}( z^i_{n+1} \mid \pi_{G}(Z^i_{n+1}) ),$
where $\pi_\mathcal{G}(Z^i_{n+1})$ denotes the (index-ordered) set of realisations $\{ z^{j}_{o} : Z^j_o \in \Pi_\mathcal{G}(Z^i_{n+1}) \}$.

To model the synchronous \ac{GDS} as a \ac{DBN}, we associate each subsystem vertex $V^i$ with a state variable $X^i_n$ and an observation variable $Y^i_n$; the parents of subsystem $V^i$ are denoted $\Pi_G(V^i)$~\cite{cliff16b}. From the dynamics~\eqref{eq:dynamics}, variables in the set $\Pi_G( X^i_{n+1} )$ come strictly from the preceding time slice, and additionally, from the measurement function~\eqref{eq:observation}, $\Pi_{G}( Y^i_{n+1} ) = X^i_{n+1}$. Thus, we can build the edge set $\mathcal{E}$ in the \ac{GDS} by means of the \ac{DBN}~\cite{cliff16b}, i.e., given an edge $X^i_n \to X^j_{n+1}$ of the \ac{DBN}, the equivalent edge $V^i \to V^j$ exists for the \ac{GDS}. The distributions for the dynamics~\eqref{eq:dynamics} and observation~\eqref{eq:observation} maps of $M$ arbitrary subsystems can therefore be factorised according to the \ac{DBN} structure such that~\cite{cliff16b}
\begin{equation} \label{eq:factored-dbn}
	p_{B}( \boldsymbol{z}_{n+1} \mid \boldsymbol{z}_{n} ) = \prod^M_{i=1} p_D( x^i_{n+1} \mid x^i_n, \langle x^{ij}_{n} \rangle_j ) \cdot p_D( y^i_{n+1} \mid x^i_{n+1} ).
\end{equation}
The goal of learning nonlinear dynamical networks thus becomes that of inferring the parent set $\Pi_G( X^{i}_n )$ for each latent variable $X^i_n$.

\section{Network Scoring Functions} \label{sec:learning-dbns}

A number of exact and approximate \ac{DBN} structure learning algorithms exist that are based on Bayesian statistics and information theory. We have shown in prior work how to compute the log-likelihood function for synchronous \acp{GDS}. In this section, we will review the literature on structure learning for \acp{DBN}, focusing on the factorised distribution in Eq.~\eqref{eq:factored-dbn}. Then, we present our proposed approach to structure learning based on conditional KL divergence.

We focus on the methods for learning the synchronous \ac{GDS} structure using the \emph{score and search} paradigm~\citep{koller09a}, which can be stated as: given a dataset $D = (\boldsymbol{y}_1, \boldsymbol{y}_2, \ldots, \boldsymbol{y}_N)$ of multivariate observations, find a \ac{DAG} $G^*$ such that
\begin{equation} \label{eq:goal}
	G^* = \argmax_{G \in \mathcal{G}} g( B : D ),
\end{equation}
where $g(B:D)$ is a scoring function measuring the degree of fitness of a candidate \ac{DAG} $G$ to the data set $D$, and $\mathcal{G}$ is the set of all \acp{DAG}. Finding the optimal graph $G^*$ in~\eqref{eq:goal} requires solutions to the two subproblems that comprise structure learning: the \emph{evaluation} problem and the \emph{identification} problem~\citep{chickering02a}. The main problem we focus on in this paper is the evaluation problem, i.e., determining a score that quantifies the quality of a graph, given data. Later we will address the identification problem by discussing the attributes of this scoring function in efficiently finding the optimal graph structure.

\subsection{Prior work}

A common approach to developing a score is to consider the posterior probability of the network structure $G$, given data $D$. Using Bayes' rule, we can express this distribution as $p( G \mid D ) \propto p( D \mid G ) p( G )$, where $p( G )$ encodes any prior assumptions we want to make about the network $G$. Thus, the problem becomes that of computing the likelihood of the data, given the model, $p( D \mid G )$. The likelihood can be written in terms of distributions over network parameters~\citep{friedman98a} $p( D \mid G ) = \int p( D \mid G, \Theta ) p( \Theta \mid G ) d\Theta$. Taking this approach, denote $\ell( \hat{\Theta}_G : D ) = \log p( D \mid G, \hat{\Theta}_G )$ as the log-likelihood function for a choice of parameters $\hat{\Theta}_G$ that maximise $p( D \mid G, \Theta )$, given a graph $G$. A number of asymptotically optimal information criterion can then be computed as a function of the log-likelihood $\ell( \hat{\Theta}_G : D )$, the model dimension (number of parameters) $C(G)$, and the dataset size $f(N)$, given by the general form~\cite{cliff16b}
\begin{equation} \label{eq:g-mdl}
	g_{\textsc{IC}}(B:D) = \ell( \hat{\Theta}_G : D ) - f(N) \cdot C(G).
\end{equation}
When $f(N)=1$, we have the \ac{AIC} score~\citep{akaike74a}, $f(N)=\log(N)/2$ is the \ac{BIC} score~\citep{schwarz78a}, and $f(N)=0$ gives the maximum likelihood score.

We have recently shown that state space reconstruction (see Appendix A) can be used to compute the log-likelihood of~\eqref{eq:factored-dbn} as a difference of conditional entropy terms~\citep{cliff16b}:
\begin{equation} \label{eq:log-ll-cond-entropy}
	\ell( \hat{\Theta}_{G} : D ) = N \cdot H( \boldsymbol{X}_n \mid \langle Y^{i,(\kappa^i)}_n \rangle ) - N \cdot \sum_{i=1}^M H( Y^i_{n+1} \mid Y^{i,(\kappa^i)}_{n}, \langle Y^{ij,(\kappa^{ij})}_n \rangle_j ),
\end{equation}
where $H( Z \mid W )$ is the entropy of variable $Z$ conditioned on $W$~\cite{mackay03a},
\begin{equation} \label{eq:cond-entropy}
	H( Z \mid W )-\sum_{ z,w } p( z, w ) \log_2 p( z \mid w ).
\end{equation}
In order to calculate the model complexity $C(G)$ for this information criterion~\eqref{eq:g-mdl}, a parametric model is required for density estimation. Thus we can not rely on non-parametric density estimators and instead must discretise the dataset to some resolution or derive a parametric model from, e.g., the physics of the phenomenon being studied.

\subsection{Proposed approach}

To overcome the issue of parameterising the distributions, in this work we consider the different problem of finding an optimal \ac{DBN} structure as searching for a parsimonious \emph{factorised distribution} that best represents the \emph{joint distribution}. \ifpaper \cite{campos06a} \else De Campos~\cite{campos06a} \fi proposes using the KL divergence as a natural information-theoretic approach to quantifying the similarity of these distributions for a \ac{BN}. We extend this approach to the \ac{DBN} structure learning problem by considering the conditional KL divergence, i.e., we compare the joint and factorised distributions of time slices, given the entire history,\footnote{\ifpaper \cite{vinh11b} \else Vinh et al.~\cite{vinh11b} \fi applied the \textsc{mit} algorithm~\citep{campos06a} to \ac{DBN} structure learning with complete data, however did not derive the results explicitly from conditional KL divergence. We show a full derivation here for the case with latent variables.}
\begin{align} \label{eq:kl-simp}
	 D_{\text{KL}} \left( p_D \parallel p_B \right) &= D_{\text{KL}} \left( p_D ( \boldsymbol{z}_{n+1} \mid \boldsymbol{z}_{n}^{(n)} ) \parallel p_B( \boldsymbol{z}_{n+1} \mid \boldsymbol{z}_{n}^{(n)} ) \right) \nonumber \\
	 &= \sum_{ \boldsymbol{z}_{n+1}, \boldsymbol{z}_{n}^{(n)} } p_{D} ( \boldsymbol{z}_{n+1}, \boldsymbol{z}_{n}^{(n)} ) \log_2 \frac{ p_{D} ( \boldsymbol{z}_{n+1} \mid \boldsymbol{z}_{n}^{(n)} ) }{ p_B ( \boldsymbol{z}_{n+1} \mid \boldsymbol{z}_n ) }.
\end{align}
Although~\eqref{eq:kl-simp} is not yet a scoring function, in Sec.~\ref{sec:scoring-functions} we present a number of scores based on this measure. First, however, we must derive a tractable form of KL divergence. Substituting the synchronous \ac{GDS} model~\eqref{eq:factored-dbn} into~\eqref{eq:kl-simp}, we get
\begin{equation} \label{eq:kl}
D_{\text{KL}} \left( p_D \parallel p_B \right) = \sum_{ \boldsymbol{z}_{n+1}, \boldsymbol{z}_{n}^{(n)} } p_{D} ( \boldsymbol{z}_{n+1}, \boldsymbol{z}_{n}^{(n)} ) \log_2 \frac{ p_{D} ( \boldsymbol{z}_{n+1} \mid \boldsymbol{z}_{n}^{(n)} ) }{ \prod^M_{i=1} p_{D}( x^i_{n+1} \mid x^i_n, \langle x^{ij}_{n} \rangle_j ) \cdot p_{D}( y^i_{n+1} \mid x^i_{n+1} ) }.
\end{equation}
Unfortunately,~\eqref{eq:kl} comprises maximum likelihood distributions with unobserved (latent) state components $\boldsymbol{x}_n$; to compute these distribution, we resort to state space reconstruction.

\section{Computing the conditional KL divergence} \label{sec:takens-embedding-for-gds}

In this section we use state space reconstruction theorems based on Takens' seminal work~\citep{takens81a} to obtain a tractable form of the conditional KL divergence~\eqref{eq:kl}. Following this, we reformulate this expression as a sum of two information-theoretic terms for use in our scoring functions (described later).

\subsection{A tractable expression via state space reconstruction}

In order to compute the distributions in~\eqref{eq:kl}, we use the Bundle Delay Embedding Theorem~\citep{stark97a} to reformulate the factorised distribution (denominator), and the Delay Embedding Theorem for Multivariate Observation Functions~\citep{deyle11a} for the joint distribution (numerator). We describe these theorems in detail in Appendix A, along with the technical assumptions required for $(f,\psi)$. The first step is to reproduce our prior result for computing the factorised distribution (denominator) in Eq.~\eqref{eq:kl}.
\begin{lemma}[\ifpaper \cite{cliff16b}\else Cliff et al.~\cite{cliff16b}\fi] \label{lem:graph-cpd}
	Given an observed dataset $D=(\boldsymbol{y}_1,\boldsymbol{y}_2,\ldots,\boldsymbol{y}_N)$, where $\boldsymbol{y}_n\in\mathbb{R}^M$, generated by a directed and acyclic synchronous \ac{GDS} $(G, \boldsymbol{x}_n, \boldsymbol{y}_n, \{ f^i \}, \{ \psi^i \})$, the 2TBN distribution can be written as
	\begin{equation} \label{eq:cpd2}
		\prod^M_{i=1} p_D( x^i_{n+1} \mid x^i_n, \langle x^{ij}_{n} \rangle_j ) \cdot p_D( y^i_{n+1} \mid x^i_{n+1} ) = \frac{ \prod_{i=1}^M p_D( y^i_{n+1} \mid y^{i,(\kappa^i)}_n,\langle y^{ij,(\kappa^{ij})}_n \rangle_j ) }{ p_D( \boldsymbol{x}_n \mid \langle y^{i,(\kappa^i)}_n \rangle ) }.
	\end{equation}
\end{lemma}

Next, we present a method for computing the joint distribution (numerator) in Lemma~\ref{lem:data-cpd}. For convenience, Lemma~\ref{lem:multivariate-takens} restates part of the delay embedding theorem in~\cite{deyle11a} in terms of subsystems of a synchronous \ac{GDS} and establishes existence of a map $\mathbf{G}$ for predicting future observations from a history of observations.
\begin{lemma} \label{lem:multivariate-takens}
	Consider a diffeomorphism $f:\mathcal{M}\to\mathcal{M}$ on a $d$-dimensional manifold $\mathcal{M}$, where the multivariate state $\boldsymbol{x}_n$ consists of $M$ subsystem states $\langle x_n^1,x_n^2,\ldots,x_n^M \rangle$. Each subsystem state $x^i_n$ is confined to a submanifold $\mathcal{M}^i\subseteq \mathcal{M}$ of dimension $d^i \leq d$, where $\sum_i d^i=d$. The multivariate observation can be estimated, for some map $\mathbf{G}$, by $\boldsymbol{y}_{n+1} = \mathbf{G}( \langle y^{i,(\kappa^i)}_{n} \rangle )$.
\end{lemma}
\begin{proof}
	We can reformulate the proof of Deyle et al.~\citep{deyle11a} in terms of subsystems. Given $M$ inhomogeneous observation functions $\langle \psi^1, \psi^2, \ldots, \psi^M \rangle$, the following map
	\begin{equation} \label{eq:phis}
		\boldsymbol{\Phi}_{f,\psi}(\boldsymbol{x}) = \langle \boldsymbol{\Phi}_{f^1,\psi^1}(\boldsymbol{x}),\boldsymbol{\Phi}_{f^2,\psi^2}(\boldsymbol{x}), \ldots,\boldsymbol{\Phi}_{f^M,\psi^M}(\boldsymbol{x}) \rangle
	\end{equation}
	is an embedding where each subsystem (local) map $\boldsymbol{\Phi}_{f^i,\psi^i}:\mathcal{M}\to\mathbb{R}^{\kappa^i}$, smoothly (at least $\mathbb{C}^2$), and, at time index $n$ is described by
	\begin{align} \label{eq:lags}
		\boldsymbol{\Phi}_{f^i,\psi^i}(\boldsymbol{x}_n) = y^{i,(\kappa^i)}_n &= \langle \psi^i \left( \boldsymbol{x}_n \right), \psi^{i} ( \boldsymbol{x}_{n-\tau} ), \psi^{i} ( \boldsymbol{x}_{n-2\tau} ), \ldots, \psi^i ( \boldsymbol{x}_{n-(k-1)\tau} ) \rangle \nonumber \\
											   			  &= \langle y^i_n, y^i_{n-\tau^i}, y^i_{n-2\tau^i}, \ldots,y^i_{n-(\kappa^i-1)\tau^i} \rangle,
	\end{align}
	where $\tau^i$ is the lag, $\kappa^i$ is the embedding dimension of the $i$th subsystem, and $\sum_i \kappa^i=2d+1$~\citep{deyle11a}.\footnote{The original proof uses positive lags for notational simplicity, however the authors note that the use of negative lags also applies, and will be used in our derivation to account for endomorphisms.} Note that, from~\eqref{eq:phis} and~\eqref{eq:lags}, we have the global map
	\begin{equation} \nonumber
		\boldsymbol{\Phi}_{f,\psi}(\boldsymbol{x}_n) = \langle y^{i,(\kappa^i)}_n \rangle = \langle y^{1,(\kappa^1)}_n, y^{2,(\kappa^2)}_n, \ldots, y^{m,(\kappa^M)}_n \rangle.
	\end{equation}
	Now, since $\boldsymbol{\Phi}_{f,\psi}$ is an embedding, it follows that the map $\mathbf{F} = \boldsymbol{\Phi}_{f,\psi} \circ f \circ \boldsymbol{\Phi}_{f,\psi}^{-1}$ is well defined and a diffeomorphism between two observation sequences $\mathbf{F} : \mathbb{R}^{2d+1} \to \mathbb{R}^{2d+1}$, i.e.,
	\begin{align}
		\langle y^{i,(\kappa^i)}_{n+1} \rangle &= \boldsymbol{\Phi}_{f,\psi} \left( \boldsymbol{x}_{n+1} \right) = \boldsymbol{\Phi}_{f,\psi} \left( f \left( \boldsymbol{x}_n \right) \right) \nonumber \\
										  &= \boldsymbol{\Phi}_{f,\psi} \left( f \left( \boldsymbol{\Phi}^{-1}_{f,\psi}\left( \langle y^{i,(\kappa^i)}_{n} \rangle \right) \right) \right) = \mathbf{F}( \langle y^{i,(\kappa^i)}_{n} \rangle ). \nonumber
	\end{align}
	The last $2d + 1$ components of $\mathbf{F}$ are trivial, i.e., the set $\langle y^{i,(\kappa^i)}_{n} \rangle$ is observed \emph{a priori}; denote the first $M$ components by $\mathbf{G}:\boldsymbol{\Phi}_{f,\psi} \to \mathbb{R}^M$, then we have that $\boldsymbol{y}_{n+1} = \mathbf{G}( \langle y^{i,(\kappa^i)}_{n} \rangle )$. 
\end{proof}

We now use the result of Lemma~\ref{lem:multivariate-takens} to obtain a computable form of the joint distribution.
\begin{lemma} \label{lem:data-cpd}
	Given an observed dataset $D=(\boldsymbol{y}_1,\boldsymbol{y}_2,\ldots,\boldsymbol{y}_N)$, where $\boldsymbol{y}_n\in\mathbb{R}^M$, generated by a discrete-time multivariate dynamical system with generic $(f,\psi)$, the joint distribution can be written as
	\begin{equation} \label{eq:cpd}
		p_D( \boldsymbol{z}_{n+1} \mid \boldsymbol{z}_{n}^{(n)} ) = \frac{p_{D} ( \boldsymbol{y}_{n+1} \mid \langle y^{i, (\kappa^i)}_n \rangle )}{p_{D} ( \boldsymbol{x}_{n} \mid \langle y^{i,(\kappa^i)}_{n} \rangle )}.
	\end{equation}
\end{lemma}
\begin{proof}
Firstly, by the chain rule
\begin{equation} \label{eq:chain}
	p_{D} ( \boldsymbol{z}_{n+1} \mid \boldsymbol{z}_{n}^{(n)} ) = p_{D} ( \boldsymbol{x}_{n+1} \mid \boldsymbol{z}_{n}^{(n)} ) \cdot p_{D} ( \boldsymbol{y}_{n+1} \mid \boldsymbol{x}_{n+1}, \boldsymbol{z}_{n}^{(n)} )
\end{equation}
Assuming we had realisations of $(\boldsymbol{x}_n,\boldsymbol{x}_{n+1})$, the probability distribution of~\eqref{eq:chain} would then be given by the product
\begin{equation} \label{eq:min-noise}
	p_{D} ( \boldsymbol{z}_{n+1} \mid \boldsymbol{z}_{n}^{(n)} ) = p_D( \boldsymbol{X}_{n+1} = f( \boldsymbol{x}_n ) \mid \boldsymbol{x}_n ) \cdot p_D( \boldsymbol{Y}_{n+1} = \psi(\boldsymbol{x}_{n+1}) \mid \boldsymbol{x}_{n+1} ).
\end{equation}
From Lemma~\ref{lem:multivariate-takens}, we have the set of equations
\begin{gather}
	\boldsymbol{x}_{n+1} = f( \boldsymbol{x}_{n} ) + \boldsymbol{\upsilon}_{f} = f \left( \boldsymbol{\Phi}_{f,\psi}^{-1} \left( \langle y^{i,(\kappa^i)}_n \rangle \right) \right) + \boldsymbol{\upsilon}_{f}, \label{eq:x-from-y} \\
	\boldsymbol{y}_{n+1} = \psi( \boldsymbol{x}_{n+1} ) + \boldsymbol{\upsilon}_{\psi} = \mathbf{G}( \langle y^{i,(\kappa^i)}_n \rangle ) + \boldsymbol{\upsilon}_{\psi}. \label{eq:y-from-y}
\end{gather}
Given the assumption of i.i.d noise on the function $f$, from~\eqref{eq:x-from-y}, we express the probability of observing $\boldsymbol{x}_{n+1}$ given by the embedding as
\begin{align} \label{eq:px-from-y}
			p_D( \boldsymbol{x}_{n+1} \mid \langle y^{i,(\kappa^i)}_n \rangle ) &= p_D( \boldsymbol{X}_{n+1} = f\left( \boldsymbol{\Phi}_{f,\psi}^{-1} \left( \langle y^{i,(\kappa^i)}_n \rangle \right) \right) \mid \langle y^{i,(\kappa^i)}_n \rangle ) \nonumber \\
			&= p_D\left( \boldsymbol{X}_n = \boldsymbol{\Phi}_{f,\psi}^{-1} \left( \langle y^{i,(\kappa^i)}_n \rangle \right) \mid \langle y^{i,(\kappa^i)}_n \rangle \right) \cdot  p_D \left( \boldsymbol{X}_{n+1} = f( \boldsymbol{x}_n ) \mid \boldsymbol{x}_n \right),
\end{align}
From our assumption that the observation noise is i.i.d or dependent only on the state $\boldsymbol{x}_{n+1}$, the probability of observing $\boldsymbol{y}_{n+1}$, from~\eqref{eq:y-from-y} is
\begin{align} \label{eq:py-from-y}
	p_D( \boldsymbol{y}_{n+1} \mid \langle y^{i,(\kappa^i)}_{n} \rangle ) &= p_D( \boldsymbol{Y}_{n+1} = \mathbf{G}( \langle y^{i,(\kappa^i)}_{n} \rangle ) \mid \langle y^{i,(\kappa^i)}_{n} \rangle ) \nonumber \\
	&= p_D( \boldsymbol{X}_{n+1} = f\left( \boldsymbol{\Phi}_{f,\psi}^{-1} \left( \langle y^{i,(\kappa^i)}_n \rangle \right) \right) \mid \langle y^{i,(\kappa^i)}_n \rangle ) \cdot p_D\left( \boldsymbol{Y}_{n+1} = \psi( \boldsymbol{x}_{n+1} ) \mid \boldsymbol{x}_{n+1} \right).
\end{align}
Substituting Eq.~\eqref{eq:px-from-y} into~\eqref{eq:py-from-y}, we have that
\begin{equation} \label{eq:final}
	p_D( \boldsymbol{x}_{n+1} \mid \boldsymbol{x}_n ) \cdot p_D( \boldsymbol{y}_{n+1} \mid \boldsymbol{x}_{n+1} ) = \frac{ p_D( \boldsymbol{y}_{n+1} \mid \langle y^{i,(\kappa^i)}_n \rangle ) }{ p_D( \boldsymbol{x}_n \mid \langle y^{i,(\kappa^i)}_n \rangle ) }
\end{equation}
Finally, substituting Eq.~\eqref{eq:final} into Eq.~\eqref{eq:min-noise} gives Eq.~\eqref{eq:cpd}.
\end{proof}
Using Lemma~\ref{lem:graph-cpd}, we can substitute~\eqref{eq:cpd2} into~\eqref{eq:kl}
\begin{equation} \label{eq:int}
 D_{\text{KL}}( p_D \parallel p_B ) = \sum_{ \boldsymbol{z}_{n+1}, \boldsymbol{z}_{n}^{(n)} } p_{D} ( \boldsymbol{z}_{n+1}, \boldsymbol{z}_{n}^{(n)} ) \log_2 \frac{ p_{D}( \boldsymbol{z}_{n+1} \mid \boldsymbol{z}_{n}^{(n)} ) \cdot p_D( \boldsymbol{x}_n \mid \langle y^{i,(\kappa^i)}_n \rangle ) }{ \prod_{i=1}^M p_D( y^i_{n+1} \mid y^{i,(\kappa^i)}_n,\langle y^{ij,(\kappa^{ij})}_n \rangle_j ) }.
\end{equation}
Then, from Lemma~\ref{lem:data-cpd}, we can substitute~\eqref{eq:cpd} into~\eqref{eq:int}, giving
\begin{align} \label{eq:kl-div}
 D_{KL}( p_D \parallel p_B ) &= \sum_{ \boldsymbol{y}_{n+1}, \langle y^{i,(\kappa^i)}_{n} \rangle } p_{D} ( \boldsymbol{y}_{n+1}, \langle y^{i,(\kappa^i)}_{n} \rangle ) \log_2 \frac{p_{D} ( \boldsymbol{y}_{n+1} \mid \langle y^{i,(\kappa^i)}_{n} \rangle ) }{ \prod_{i=1}^M p_D( y^i_{n+1} \mid y^{i,(\kappa^i)}_n,\langle y^{ij,(\kappa^{ij})}_n \rangle_j ) } \nonumber \\
\end{align}
Given all variables in~\eqref{eq:kl-div} are observed, it is straightforward to compute KL divergence; however, as we will see, it is more convenient to express~\eqref{eq:kl-div} as a function of information-theoretic measures.

\subsection{Information-theoretic interpretation}

Before presenting the main theorem of the paper, we first introduce the concepts of collective transfer entropy and stochastic interaction. Transfer entropy detects the directed exchange of information between random processes by marginalising out common history and static correlations between variables; it is thus considered a measure of information transfer within a system~\citep{schreiber00a}. The collective transfer entropy computes the information transfer between a set of $M$ source processes and a single destination process~\citep{lizier10b}. Consider the set $\boldsymbol{Y}= \{ Y^i \}$ of source processes. We can compute the collective transfer entropy from $\boldsymbol{Y}$ to the destination process $X$ as a function of conditional entropy~\eqref{eq:cond-entropy} terms
\begin{align} \label{eq:te}
	T_{\boldsymbol{Y} \to X} &= H \left( X_{n+1} \mid X^{(\kappa^i)}_{n} \right) - H\left( X_{n+1} \mid X^{(\kappa^i)}_{n}, \langle Y^{i,(\kappa^{i})}_{n} \rangle \right) 
\end{align}
Stochastic interaction measures the complexity of dynamical systems by quantifying the excess of information processed, in time, by the system beyond the information processed by each of the nodes~\citep{ay03a,ay03b,edlund11a}. Using the same notation, stochastic interaction of the collection of processes $\boldsymbol{Y}$ is
\begin{equation} \label{eq:stochastic-interaction}
	S_{\boldsymbol{Y}} = - H \left( \boldsymbol{Y}_{n+1} \mid \langle Y_{n}^{i,(\kappa^i)} \rangle \right) + \sum_{i=1}^M H \left( Y^i_{n+1} \mid Y^{i,(\kappa^i)}_{n} \right).
\end{equation}
Note that the original definition assumed a first-order Markov process~\citep{ay03a}, and here we have extended stochastic interaction to arbitrary $\kappa$-order Markov chains. Given these definitions, we have the following result.

\begin{theorem} \label{theorem:divergence}
Consider a discrete-time multivariate dynamical system with generic $(f,\psi)$ represented as a directed and acyclic synchronous \ac{GDS}  $( G, \boldsymbol{x}_n, \boldsymbol{y}_n, \{ f^{i} \}, \{ \psi^i \})$ with $M$ subsystems. The KL divergence $D_{\text{KL}}( p_D \parallel p_B )$ of a candidate graph $G$ from the observed dataset $D = (\boldsymbol{y}_1, \boldsymbol{y}_2, \ldots, \boldsymbol{y}_N )$ is given by the difference between stochastic interaction and collective transfer entropy, i.e.,
\begin{equation} \label{eq:kl-te}
 	D_{\text{KL}}( p_D \parallel p_B ) = S_{\boldsymbol{Y}} - \sum_{i=1}^{m} T_{ \{ Y^{ij} \}_j \to Y^i}.
\end{equation}
\end{theorem}
\begin{proof}
	We can reformulate~\eqref{eq:kl-div} as
	\begin{align} \label{eq:kl-div-split}
	 D_{\text{KL}}( p_D \parallel p_B ) &= \sum_{ \boldsymbol{y}_{n+1}, \langle y^{i,(\kappa^i)}_{n} \rangle } p_{D} ( \boldsymbol{y}_{n+1}, \langle y^{i,(\kappa^i)}_{n} \rangle ) \log_2 p_{D} ( \boldsymbol{y}_{n+1} \mid \langle y^{i,(\kappa^i)}_{n} \rangle ) \nonumber \\
		 &\hspace{10mm} - \sum_{ \boldsymbol{y}_{n+1}, \langle y^{i,(\kappa^i)}_{n} \rangle } p_{D} ( \boldsymbol{y}_{n+1}, \langle y^{i,(\kappa^i)}_{n} \rangle )\log_2 \prod_{i=1}^M p_{D} ( y^i_{n+1} \mid y^{i,(\kappa^i)}_{n}, \langle y^{ij,(\kappa^{ij})}_{n} \rangle_j ).
	\end{align}
	Splitting the latter term in~\eqref{eq:kl-div-split} into subsystems without a parent set $\Pi_G(V^i)=\emptyset$ and subsystems with a parent set $\Pi_G(V^i)\neq\emptyset$, we get a function of conditional entropy~\eqref{eq:cond-entropy} terms
	\begin{align} \label{eq:kl-div-ce}
	   D_{\text{KL}}( p_D \parallel p_B ) =&  - H ( \boldsymbol{Y}_{n+1} \mid \langle Y^{(\kappa^i)}_{n} \rangle ) \nonumber \\
					 &\hspace{10mm} + \sum_{\tiny{\substack{i=1,\\\Pi_{G}(V^i) = \emptyset}}}^{M} H( Y^i_{n+1} \mid Y^{i,(\kappa^i)}_{n} ) + \sum_{\tiny{\substack{i=1,\\\Pi_{G}(V^i) \neq \emptyset}}}^{M} H( Y^i_{n+1} \mid Y^{i,(\kappa^i)}_{n}, \langle Y^{ij,(\kappa^{ij})}_{n} \rangle_j )
	\end{align}
	 Then, by adding $\sum_{\tiny{\substack{i=1,\Pi_{G}(V^i) \neq \emptyset}}}^{M} H( Y^i_{n+1} \mid Y^{i,(\kappa^i)}_{n})$ to the second term and subtracting it from the last, we can rewrite KL divergence~\eqref{eq:kl-div-ce} in terms of collective transfer entropy~\eqref{eq:te} and stochastic interaction~\eqref{eq:stochastic-interaction}
	\begin{align} \label{eq:dist-cond-entropy}
	   	D_{\text{KL}}( p_D \parallel p_B ) &= - H ( \boldsymbol{Y}_{n+1} \mid \langle Y^{(\kappa^i)}_{n} \rangle ) + \sum_{i=1}^M H ( Y_{n+1}^i \mid Y^{i,(\kappa^i)}_{n} ) \nonumber \\
												   	 &\hspace{10mm} - \sum_{\tiny{\substack{i=1,\\\Pi_{G}(V^i) \neq \emptyset}}}^{M} \left[ H ( Y^i_{n+1} \mid Y^{i,(\kappa^i)}_{n} ) - H( Y^i_{n+1} \mid Y^{i,(\kappa^i)}_{n}, \langle Y^{ij,(\kappa^{ij})}_{n} \rangle_j ) \right] \nonumber \\
	   	 &= S_{\boldsymbol{Y}} - \sum_{\tiny{\substack{i=1,\\\Pi_{G}(V^i) \neq \emptyset}}}^{M} T_{ \{Y^{ij}\}_j \to Y^i}.
	\end{align}
	Note that, in~\eqref{eq:dist-cond-entropy}, we can remove the specification that the transfer entropy sum is over non-empty parent sets $\Pi_G( V^i ) \neq \emptyset$ since transfer entropy is a measure, and therefore, for any $Y^i$, $T_{\emptyset \to Y^i} = 0$, so $\sum_{\tiny{\substack{i=1,\Pi_{G}(V^i) \neq \emptyset}}}^{M}  T_{ \{Y^{ij}\}_j \to Y^i} = \sum_{i=1}^M  T_{ \{Y^{ij}\}_j \to Y^i}$, giving~\eqref{eq:kl-te}.
\end{proof}

\section{Scoring functions based on transfer entropy} \label{sec:scoring-functions}

There are a number of ways to score a candidate graph based on Theorem~\ref{theorem:divergence}. Here we present a corollary of this theorem from which we derive two scores: (1) \emph{transfer entropy with analytic independence tests}~(\textsc{tea}), and (2) \emph{transfer entropy with empirical independence tests}~(\textsc{tee}). First, we will show that a maximum likelihood-based approach is insufficient for structure learning.

\subsection{The maximum likelihood approach}
A common method to derive a score is to minimise the KL divergence between graph and empirical distributions~\citep{lam94a,friedman96a}. This score follows naturally from Theorem~\ref{theorem:divergence}. The following corollary shows that in practice it suffices to maximise the collective transfer entropy alone in order to minimise KL divergence for a synchronous \ac{GDS}.
\begin{corollary} \label{cor:min-kl}
The minimum KL divergence of a candidate graph $G$ from the empirical dataset $D$ is equivalent to the maximum transfer entropy graph, i.e.,
\begin{equation} \label{eq:argmin}
	\argmin_{G\in\mathcal{G}} D_{\text{KL}}( p_D \parallel p_B ) = \argmax_{G\in\mathcal{G}} \sum_{i=1}^{m} T_{ \{Y^{ij}\}_j \to Y^i}.
\end{equation}
\end{corollary}
\begin{proof}
The stochastic interaction term $S_{\boldsymbol{Y}}$ in~\eqref{eq:kl-te} is defined in terms of persistent variables, i.e., each variable $Y^i_{n+1}$ is conditioned only on its own past $Y^{i,(\kappa^i)}_n$. Stochastic interaction is therefore constant, given a constant vertex set $\mathcal{V}$, time delay $\tau$ and embedding dimension $\kappa$ and is thus unaffected by the parent set $\Pi_G(V^i)$ of a variable. This is evident in~\eqref{eq:dist-cond-entropy}, where only the latter sum depends on the parent set $\Pi_G(V^i)$.
As a result, stochastic interaction $S_{\boldsymbol{Y}}$ does not depend on the graph $G$ being considered, and, therefore
\begin{equation} \label{eq:min}
	\min_{G\in\mathcal{G}} D_{\text{KL}}( p_D \parallel p_B ) = \min_{G\in\mathcal{G}} \left( S_{\boldsymbol{Y}} - \sum_{i=1}^{m}T_{ \{Y^{ij}\}_j \to Y^i} \right) = S_{\boldsymbol{Y}} - \max_{G\in\mathcal{G}} \left( \sum_{i=1}^{m} T_{ \{Y^{ij}\}_j \to Y^i} \right).
\end{equation} %
Taking instead the arguments of the optima in~\eqref{eq:min} gives~\eqref{eq:argmin}.
\end{proof}
From Corollary~\ref{cor:min-kl}, a naive score can be defined as
\begin{equation} \label{eq:g-kl}
	g_{\textsc{te}}(B : D) = \sum_{i=1}^{m} T_{ \langle Y^{ij} \rangle_j \to Y^i}.
\end{equation}
However, this score is insufficient. Maximising collective transfer entropy will always yield a complete graph. For example, let $\mathcal{Y} = \{Y^1,Y^2,\ldots,Y^M\}$, then, for any $Y^i,Y^k \in \mathcal{Y}$ and $Y^j \in \mathcal{Y} \setminus \{Y^i, Y^k\}$,
\begin{gather}
	H( Y^i_{n+1} \mid Y^j_n \cup Y^k_n ) \leq H( Y^i_{n+1} \mid Y^j_n ) \nonumber \\
	\therefore T_{Y^{j} \cup Y^k\to Y^i} \geq T_{Y^{j} \to Y^i}. \nonumber
\end{gather}
The sum of transfer entropy in~\eqref{eq:g-kl} is therefore strictly non-decreasing when including more variables in a parent set. Further, since observations are taken from a finite number of samples $N$, a non-zero bias of conditional entropy is likely to result even in the absence of dependence, particularly under noisy observations.

\subsection{Penalising transfer entropy by independence tests}

\begin{figure}[t]
	\centering
	\subcaptionbox{\textsc{tea} penalty\label{fig:tea}}
	{\includegraphics[width=.4\columnwidth, clip=true, trim=10mm 60mm 18mm 90mm]{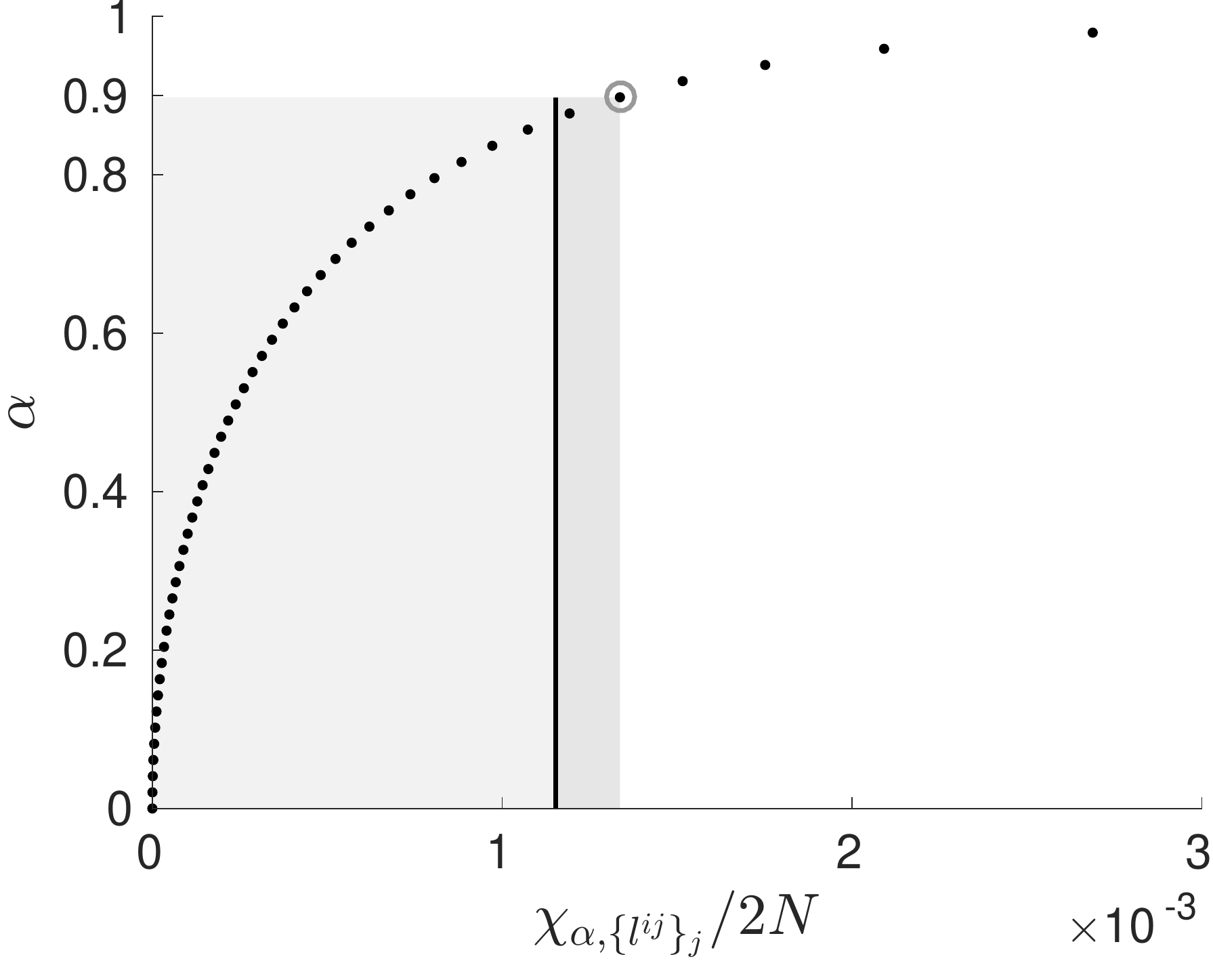}}
	\hspace{15mm}
	\subcaptionbox{\textsc{tee} penalty\label{fig:tee}}
	{\includegraphics[width=.4\columnwidth, clip=true, trim=10mm 60mm 18mm 90mm]{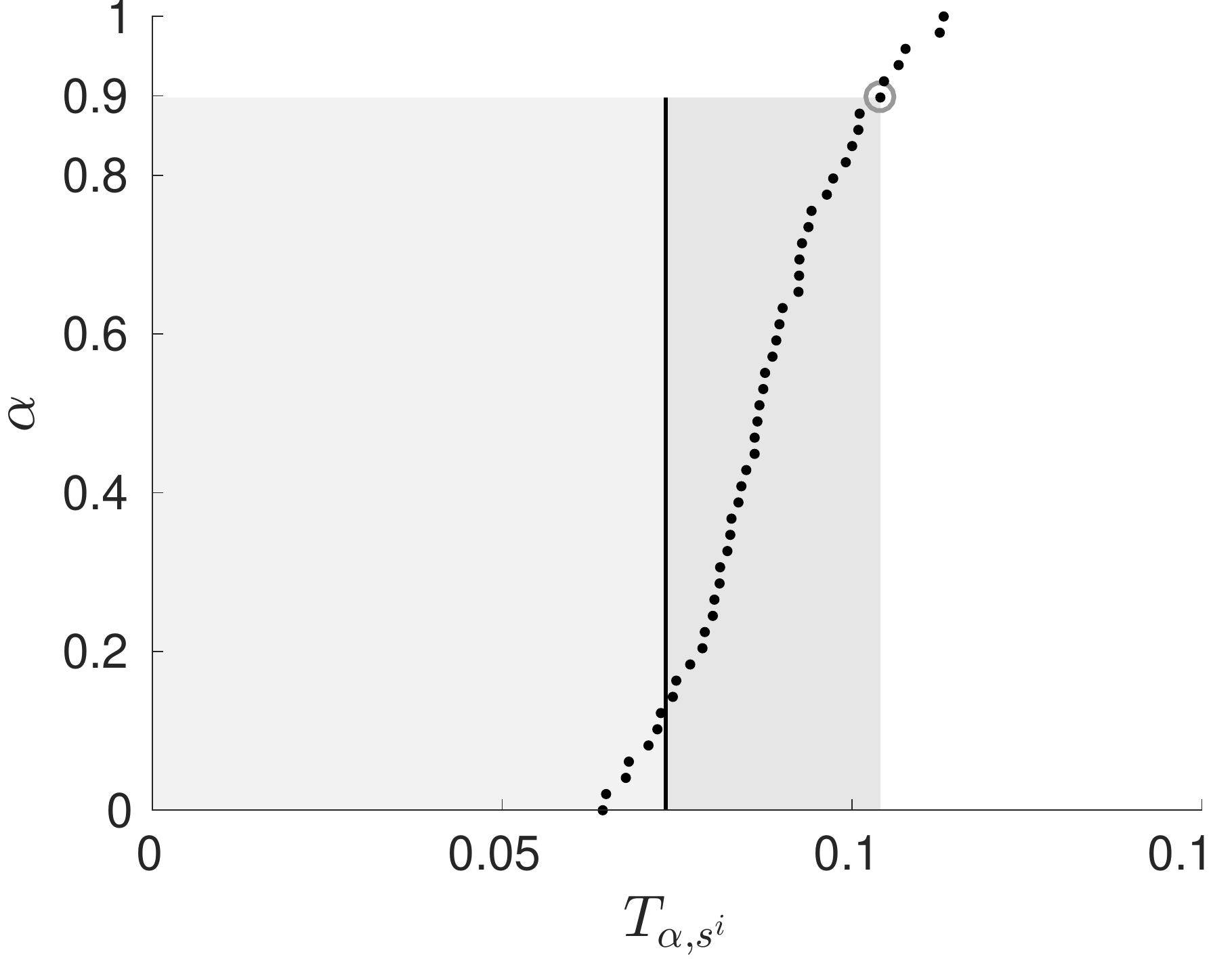}}
	\caption{Distributions of the~\ref{fig:tea} \textsc{tea} penalty function~\eqref{eq:g-tea} and the~\ref{fig:tee} \textsc{tee} penalty function~\eqref{eq:g-tea}. Both distributions were generated by observing the outcome of $1000$ samples from two Gaussian variables with a correlation of $0.05$. The figures illustrate: the distribution as a set of 100 sampled points (black dots); the area considered independent (grey regions); the measured transfer entropy (black line); and the difference between measurement and penalty term (dark grey region). Both tests use a value of $\alpha = 0.9$ (a $p$-value of $0.1$). The distribution in Fig.~\ref{fig:tea} was estimated by assuming variables were linearly-coupled Gaussians, and the distribution in Fig.~\ref{fig:tee} was computed via a kernal box method (computed by the JIDT, see~\citep{lizier14a} for details).}
	\label{fig:scores}
\end{figure}

%

Building on the maximum likelihood score~\eqref{eq:g-kl}, we propose to use independence tests to define two scores of practical value. Here, we draw on the result of \ifpaper	\citet{campos06a}, \else de Campos~\cite{campos06a}, \fi who derived a scoring function for \ac{BN} structure learning based on conditional mutual information and statistical significance tests, called \textsc{mit} (mutual information tests). The central idea is to use collective transfer entropy $T_{ \langle Y^{ij} \rangle_j \to Y^i}$ to measure the degree of interaction between each subsystem $V^i$ and its parent subsystems $\Pi_G( V^i )$, but also to penalise this term with a value based on significance testing. As with the \textsc{mit} score, this gives a principled way to re-scale the transfer entropy when including more edges in the graph.

To develop our scores, we form a \emph{null hypothesis} $H_0$ that there is no interaction $T_{ \langle Y^{ij} \rangle_j \to Y^i}$, and then compute a test statistic to penalise the measured transfer entropy. To compute the test statistic, it is necessary to  consider the measurement distribution in the case where the hypothesis is true. Fortunately, in the case of discrete and linear-Gaussian systems, the distribution $2 N T_{ \langle Y^{ij} \rangle_j \to Y^i}$ is known to asymptotically approach the $\chi^2$-distribution~\citep{barnett12a}. Since this distribution is a function of the parents of $Y^i$, we let it be described by the function $\chi^2( \{l^{ij}\}_j )$. Now, given this distribution, we can fix some \emph{confidence level} $\alpha$ and determine the value $\chi_{\alpha,\{l^{ij}\}_j}$ such that $p( \chi^2( \{l^{ij}\}_j ) \leq \chi_{\alpha,\{l^{ij}\}_j} )$. This represents a conditional independence test: if $2 N T_{ \langle Y^{ij} \rangle_j \to Y^i} \leq \chi_{\alpha,\{l^{ij}\}_j}$, then we accept the hypothesis of conditional independence between $Y^i$ and $\langle Y^{ij} \rangle_j$; otherwise, we reject it. We express this idea as the \textsc{tea} score:
\begin{equation} \label{eq:g-tea}
	g_{\textsc{tea}}( B : D ) = \sum_{i=1}^M \left( 2 N T_{ \{Y^{ij}\}_j \to Y^i} - \chi_{\alpha,\{l^{ij}\}_j} \right).
\end{equation}

We can derive a more general form of the \textsc{tea} score~\eqref{eq:g-tea} via \emph{surrogate} measurements $T_{ \langle Y^{ij} \rangle_j^s \to Y^i}$ under the assumption of $H_0$~\cite{lizier14a}. This same technique has been used by \ifpaper \cite{lizier12c}\else Lizier and Rubinov~\cite{lizier12c} \fi to derive a greedy structure learning algorithm for effective network analysis. Here, $\langle Y^{ij} \rangle_j^s$ are surrogate sets of variables for $\langle Y^{ij} \rangle_j$, which have the same statistical properties as $\langle Y^{ij} \rangle_j$, but the correlation between $\langle Y^{ij} \rangle_j^s$ and $Y^i$ is removed. Let the distribution of these surrogate measurements be represented by some general function $T(s^i)$, and note that for the systems described for the \textsc{tea} score~\eqref{eq:g-tea}, we could compute $T(s^i)$ analytically as an independent set of $\chi^2$-distributions $\chi^2( \{l^{ij}\}_j )$. Where no analytic distribution is known, we use a resampling method (i.e., permutation or bootstrapping), creating a large number of surrogate time-series pairs $\{\langle Y^{ij} \rangle_j^s, Y^i\}$ by shuffling (for permutations, or redrawing for bootstrapping) the samples of $Y^i$ and computing a population of $T_{ \langle Y^{ij} \rangle_j^s \to Y^i}$. As with the \textsc{tea} score, we fix some confidence level $\alpha$ and determine the value $T_{\alpha,s^i}$, such that $p(T(s^i) \leq T_{\alpha,s^i}) = \alpha$. This results in the \textsc{tee} scoring function as
\begin{equation} \label{eq:g-tee}
	g_{\textsc{tee}}( B : D ) = \sum_{i=1}^M \left( T_{ \{Y^{ij}\}_j \to Y^i} - T_{\alpha,s^i} \right).
\end{equation}
We can obtain the value $T_{\alpha,s^i}$ by (1) drawing $N^s$ samples $T_{ \langle Y^{ij} \rangle_j^s \to Y^i}$ from the distribution $T(s^i)$ (by permutation or bootstrapping), (2) fixing $\alpha \in \{ 0, 1/N^s, 2/N^s, \ldots, 1 \}$, then (3) taking $T_{\alpha,s^i}$ such that
\begin{equation} \nonumber
	\alpha = \frac{1}{N^s} \sum_{T_{ \{Y^{ij}\}_j \to Y^i}} \mathbbm{1}_{T_{ \{Y^{ij}\}_j^s \to Y^i} \leq T_{\alpha,s^i}}.
\end{equation}
We can alternatively limit the number of surrogates $N_s$ to $\lceil \alpha / (1-\alpha) \rceil$ and take the maximum as $T_{\alpha,s^i}$~\cite{kantz04a}, however taking a larger number of surrogate $N_s$ will improve the validity of the distribution $T(s^i)$.

\subsection{Analysis of the scores} \label{sec:ident}

Given the \textsc{tea} and \textsc{tea} scoring functions, the optimal graph $G^*$ can be found using any search procedure over DAGs. Exhaustive search, where DAGs are enumerated and scored, is intractable because the search space is super-exponential in the number of variables (about $2^{\mathcal{O}(M^2)}$). It is therefore common to employ local search methods such as greedy hill climbing, basin flooding and tabu search~\citep{koller09a}. In this section, we discuss two properties of the scoring functions that facilitate these search procedures: \emph{decomposability} and \emph{score-equivalence}.

A decomposable score is a sum of local scores that depend only on a variable and its parents, i.e.,
\begin{gather}
	g( B : D ) = \sum_{i=1}^M g( V^i, \Pi_G( V^i ) : D ), \nonumber \\
	g( V^i, \Pi_G( V^i ) : D ) = g( V^i, \Pi_G( V^i ) : N^{D}_{V^i, \Pi_G( V^i )} ), \nonumber
\end{gather}
where $N^{D}_{V^i, \Pi_G( V^i )}$ are sufficient statistics for the set of variables $V^i \cup \Pi_G( V^i )$ in $D$~\citep{campos06a}. Given the independent sums in~\eqref{eq:g-tea} and~\eqref{eq:g-tee}, the \textsc{tea} and \textsc{tee} scoring functions are decomposable. Further, the \textsc{tea} score~\eqref{eq:g-tea} can be decomposed as a sum of conditional mutual information tests, i.e.,
\begin{equation} \nonumber
	g_{\textsc{tea}}( B : D ) = \sum_{i=1}^M \left( 2 N T_{ \{Y^{ij}\}_j \to Y^i} - \sum_{j=1}^{p^i} \chi_{\alpha,l^{ij}} \right),
\end{equation}
where $p^i$ is the number of parents of subsystem $V^i$. This approach is more efficient as it allows for caching the results of $\chi_{\alpha,\{l^{ij}\}_j}$ incrementally~\citep{campos06a}. Note that although any decomposition of collective transfer entropy yields the same value, the ordering of conditioning on the variables $Y^{ij}$ in the penalty term affects the score. This issue can be resolved by penalising the score conservatively by using the maximum permutation of the $\chi_{\alpha,\{\l^{ij}\}_j}$ value; an in-depth explanation of this approach can be found in \ifpaper \citep{campos06a} \else de Campos'~\cite{campos06a} \fi discussion of the maximum penalty permutation (Theorem~2) and Shur-concavity (Theorem~3) of the penalty term.

Score-equivalence in \ac{BN} structure learning simplifies the evaluation and identification problems by constraining the search space to a set of \emph{essential graphs}, which is a set of \emph{equivalence classes} over \acp{DAG}~\citep{chickering02a}. Because \textsc{tea} and \textsc{tee} are specific cases of the \textsc{mit} score~\citep{campos06a}, they are \emph{not} score-equivalent. However, they do satisfy the less demanding property of equivalence in the space of \acp{RPDAG}~\citep{acid03a}. Thus, these scoring functions assign the same value to all \acp{DAG} that are represented by the same \acp{RPDAG}. With a decomposable scoring function, searching in the space of \acp{RPDAG} is more efficient than searching through essential graphs, and has been shown to yield  better local optima than other local search techniques in practice~\citep{acid03a}.




\ifpaper

\section{Applications}

\subsection{Coupled logistic maps}

\oc{~\cite{kantz04a} gives a great reason for using map-like (i.e., discrete) dynamical systems to model continuous systems (normally because we take measurements at discrete intervals anyway).}

In order to evaluate our approach, we use a simple but expressive model for coupled dynamical systems called \emph{coupled maps}. Coupled maps are a generalisation of the coupled map lattice model where the topology need not be a lattice structure. These types of systems are ubiquitous in literature as a toy example for complex, chaotic behaviour arising from simple, non-linear dynamical equations.

To simulate the coupled map dynamics: denote the number of parents of subsystem $i$ as $p^i = |\Pi_{\mathcal{G}}(V^i)|$; then, we expand the local map $f^i$ as
\begin{equation} \label{eq:coupled-map}
	x_{n+1}^i = (1-\varepsilon)g^i(x^i_n) + (\varepsilon/p^i) \sum_{j=1}^{p^i} g^i(x^{ij}_n)  + \upsilon_{f^i},
\end{equation}
where $0 < \varepsilon < 1$ and for $g^i(z)$ we can use any map with chaotic behaviour.

\oc{In my understanding we can use any map from the \href{https://en.wikipedia.org/wiki/List_of_chaotic_maps}{Wiki page on chaotic maps} to simulate toy examples. \cite{kantz04a} also show some nice real world examples of biological, neurological and physical (laser data) systems}

Here, we use the logistic map $$g^i(z) = r z ( 1 - z ).$$
To observe the system, we can simply map a multivariate state to the real line, i.e., let the multivariate state be defined in terms of $d^i$ components $x^i_n = \langle x^i_n[0], x^i_n[1], \ldots, x^i_n[d^i] \rangle$, then we can define the observable as
\begin{equation} \nonumber
	y^i_{n+1} = x^i_{n+1}[0] + \nu_{{\psi^i}}.
\end{equation}
This approach can be considered a mean-field-type global coupling, where the \oc{\ldots}

\begin{figure}
	\centering
	\hfill
	\begin{subfigure}[b]{.32\textwidth}
		\centering
		\includegraphics[width=\textwidth, clip=true, trim=35mm 115mm 40mm 120mm]{../figs/CLM_N10000_M3_time_v1.pdf}\\
		\includegraphics[width=\textwidth, clip=true, trim=50mm 95mm 55mm 100mm]{../figs/CLM_N10000_M3_phase_v1.pdf}
		\caption{}
	\end{subfigure}\hfill
	\begin{subfigure}[b]{.32\textwidth}
		\centering
		\includegraphics[width=\textwidth, clip=true, trim=35mm 115mm 40mm 120mm]{../figs/CLM_N10000_M3_time_v2.pdf}\\
		\includegraphics[width=\textwidth, clip=true, trim=50mm 95mm 55mm 100mm]{../figs/CLM_N10000_M3_phase_v2.pdf}
		\caption{}
	\end{subfigure}\hfill
	\begin{subfigure}[b]{.32\textwidth}
		\centering
		\includegraphics[width=\textwidth, clip=true, trim=35mm 115mm 40mm 120mm]{../figs/CLM_N10000_M3_time_v3.pdf}\\
		\includegraphics[width=\textwidth, clip=true, trim=50mm 95mm 55mm 100mm]{../figs/CLM_N10000_M3_phase_v3.pdf}
		\caption{}
	\end{subfigure}
	\caption{Simulated coupled logistic maps. \emph{Top row}: first 200 sequential measurements of the coupled maps (from left to right, vertices $v_1$, $v_2$, and $v_3$). \emph{Bottom row}: phase portrait of the corresponding observations with time delay $\tau=1$ and embedding dimension $\kappa=2$.}
	\label{fig:coupled-logistic-map}
\end{figure}


\subsection{Coupled R\"{o}ssler-Lorenz systems}

In this section we take a number of coupled continuous dynamical systems\ldots

\fi

\section{Discussion and future work}


We have presented a principled method to learn the structure of a synchronous \ac{GDS} based on collective transfer entropy and independence tests. We derived this method analytically by reformulating the KL divergence of factorised from joint distributions of a network, which Theorem~\ref{theorem:divergence} shows can be computed in terms of stochastic interaction and transfer entropy. We arrived at this result by first reconsidering the \ac{GDS} as a \ac{DBN}, and then employed generalised versions of Takens' embedding theorem to compute densities comprising hidden and observed variables.



The decomposition of KL divergence in Theorem~\ref{theorem:divergence} captures an interesting parallel between fully observable systems and partially observable systems. \ifpaper \cite{campos06a} \else De Campos~\cite{campos06a} \fi showed previously that the KL divergence in a fully observable system is given by the difference between multi-information~\citep{studeny98a} and mutual information.\footnote{Although it is not derived in~\cite{campos06a}, it is trivial to show the first two terms constitute multi-information.} Specifically, a condition for generalised Takens' theorems to hold is that the observation functions $\{\psi^i\}$ are injective~\citep{deyle11a,stark97a}. We conjecture that if the functions are also surjective (i.e., there is a one-to-one mapping between state and observation), the embedding dimension would reduce to unity and we would arrive at the \textsc{mit} scoring function.

In Corollary~\ref{cor:min-kl}, we have shown that, under certain circumstances, maximising collective transfer entropy minimises the KL divergence of a model from the true distribution. KL divergence is related to model encoding, which is a fundamental measure used in complex systems analysis. Our result, therefore, has potential implications to other areas of complex systems research. For example, the notion of equivalence classes in \ac{BN} structure learning should lend some insight into the area of effective network analysis~\citep{sporns04a,park13a}. We believe the concepts of effective networks referred to in complex systems literature can be unified with essential graphs and \acp{RPDAG}. This would allow for a more rigorous definition of effective networks and a benchmark for analysing the efficacy of an algorithm to reconstruct these networks.

We have presented the \textsc{tea}~\eqref{eq:g-tea} and \textsc{tee}~\eqref{eq:g-tee} scores above based on the \textsc{mit} scoring function~\citep{campos06a}. These scoring functions, however, could be considered to be a generalisation of \textsc{mit}. There are numerous approaches to recover the time delay $\tau$ and embedding dimension $\kappa$ for use in transfer entropy~\cite{ragwitz02a,small04a}. Given a system of fully observed variables, these criteria should optimally select no embedding dimension or time delay, and thus as a special case of our scores we obtain the \textsc{mit} algorithm with time-lagged mutual information.


\ifpaper
	\acks{
\else
	\section*{Acknowledgements}
\fi
	This work was supported in part by the Australian Centre for Field Robotics; the New South Wales Government; and the Faculty of Engineering \& Information Technologies, The University of Sydney, under the Faculty Research Cluster Program. Special thanks to Joseph Lizier, J\"{u}rgen Jost, and Wolfram Martens for their incite in regards to dynamical systems.
\ifpaper
	}
\fi

\appendix
\section*{Appendix A. Embedding theory}

We refer here to embedding theory as the study of inferring the (hidden) state $\boldsymbol{x}_n \in \mathcal{M}$ of a dynamical system from a sequence of observations $y_n \in \mathbb{R}$. This section will cover reconstruction theorems that define the conditions under which we can use delay embeddings for recovering the original dynamics $f$ from this observed time series.

In differential topology, an \emph{embedding} refers to a smooth map $\boldsymbol{\Phi}: \mathcal{M} \to \mathcal{N}$ between manifolds $\mathcal{M}$ and $\mathcal{N}$ if it maps $\mathcal{M}$ diffeomorphically onto its image. In Takens seminal work on turbulent flow~\cite{takens81a}, he proposed a map $\boldsymbol{\Phi}_{f,\psi}: \mathcal{M} \to \mathbb{R}^\kappa$, that is composed of delayed observations, can be used to reconstruct the dynamics for typical $(f,\psi)$. That is, fix some $\kappa$ (the \emph{embedding dimension}) and $\tau$ (the \emph{time delay}), the \emph{delay embedding map}, given by
\begin{equation} \label{eq:delay-embedding-map}
	\boldsymbol{\Phi}_{f,\psi}( \boldsymbol{x}_n ) = y^{(\kappa)}_n = \langle y_{n}, y_{n+\tau}, y_{n+2\tau},\ldots, y_{n+(\kappa-1)\tau} \rangle,
\end{equation}
is an embedding. More formally, denote $\boldsymbol{\Phi}_{f,\psi}$, $\mathcal{D}^r(\mathcal{M},\mathcal{M})$ as the space of $C^r$-diffeomorphisms on $\mathcal{M}$ and $C^r(\mathcal{M},\mathbb{R})$ as the space of $C^r$-functions on $\mathcal{M}$, then the theorem can be expressed as follows.
\begin{theorem}[Delay Embedding Theorem for Diffeomorphisms~\cite{takens81a}] \label{th:delay-embedding-diffeo}
	Let $\mathcal{M}$ be a compact manifold of dimension $d \geq 1$. If $\kappa \geq 2d + 1$ and $r \geq 1$, then there exists an open and dense set $(f,\psi) \in \mathcal{D}^r(\mathcal{M},\mathcal{M})\times C^r(\mathcal{M},\mathbb{R})$ for which the map $\boldsymbol{\Phi}_{f,\psi}$ is an embedding of $\mathcal{M}$ into $\mathbb{R}^{\kappa}$. 
\end{theorem}
The implication of Theorem~\ref{th:delay-embedding-diffeo} is that, for typical $(f,\psi)$, the image $\boldsymbol{\Phi}_{f,\psi}(\mathcal{M})$ of $\mathcal{M}$ under the delay embedding map $\mathbf{\Phi}_{f,\psi}$ is completely equivalent to $\mathcal{M}$ itself, apart from the smooth invertible change of coordinates given by the mapping $\boldsymbol{\Phi}_{f,\psi}$. An important consequence of this result is that we can define a map $\mathbf{F} = \boldsymbol{\Phi}_{f,\psi} \circ f \circ \boldsymbol{\Phi}_{f,\psi}^{-1}$ on $\boldsymbol{\Phi}_{f,\psi}$, such that $ y^{(\kappa)}_{n+1} = \mathbf{F}( y^{(\kappa)}_n ) $~\cite{stark97a}. The bound for the open and dense set referred to in Theorem~\ref{th:delay-embedding-diffeo} is given by a number of technical assumptions. Denote $(Df)_{\boldsymbol{x}}$ as the derivative of function $f$ at a point $\boldsymbol{x}$ in the domain of $f$. The set of periodic points $A$ of $f$ with period less than $\tau$ has finitely many points. In addition, the eigenvalues of $(Df)_{\boldsymbol{x}}$ at each $\boldsymbol{x}$ in a compact neighbourhood $A$ are distinct and not equal to 1.

Importantly, Theorem~\ref{th:delay-embedding-diffeo} was established for diffeomorphisms $\mathcal{D}^r$; by definition the dynamics are thus invertible in time. So the time delay $\tau$ in~\eqref{eq:delay-embedding-map} can be either positive (delay lags) or negative (delay leads). Takens later proved a similar result for endomorphisms, i.e., non-invertible maps that restricts the time delay to a negative integer. Denote by $\mathcal{E}(\mathcal{M},\mathcal{M})$ the set of the space of $\mathcal{C}^r$-endomorphisms on $\mathcal{M}$, then the reconstruction theorem for endomorphisms can be expressed as the following.
\begin{theorem}[Delay Embedding Theorem for Endomorphisms~\cite{takens02a}] \label{th:delay-embedding-endo}
	Let $\mathcal{M}$ be a compact $m$ dimensional manifold. If $\kappa \geq 2d + 1$ and $r \geq 1$, then there exists an open and dense set $(f,\psi) \in \mathcal{D}^r(\mathcal{M},\mathcal{M})\times C^r(\mathcal{M},\mathbb{R})$ for which there is a map $\pi_\kappa : \mathcal{X}_\kappa \to \mathcal{M}$ with $\pi_\kappa \mathbf{\Phi}_{f,\psi} = f^{\kappa-1}$. Moreover, the map $\pi_\kappa$ has bounded expansion or is Lipschitz continuous.
\end{theorem}
As a result of Theorem~\ref{th:delay-embedding-endo}, a sequence of $\kappa$ successive measurements from a system determines the system state \emph{at the end} of the sequence of measurements~\cite{takens02a}. That is, there exists an endomorphism $\mathbf{F} = \boldsymbol{\Phi}_{f,\psi} \circ f \circ \boldsymbol{\Phi}_{f,\psi}^{-1}$ to predict the next observation if one takes a negative time (lead) delay $\tau$ in~\eqref{eq:delay-embedding-map}.

In this work, we consider two important generalisations of the Delay Embedding Theorem~\ref{th:delay-embedding-diffeo}. Both of these theorems follow similar proofs to the original and have thus been derived for diffeomorphisms, not endomorphisms. However, encouraging empirical results in~\cite{schumacher15a} support the conjecture that they can both be generalised to the case of endomorphisms by taking a negative time delay, as is done in Theorem~\ref{th:delay-embedding-endo} above.

The first generalisation is by Stark et al.~\cite{stark97a} and deals with a skew-product system. That is, $f$ is now forced by some second, independent system $g:\mathcal{N}\to\mathcal{N}$. The dynamical system on $\mathcal{M}\times\mathcal{N}$ is thus given by the set of equations
\begin{equation}
	x_{n+1} = f(x_n, \omega_n), \hspace{5mm} \omega_{n+1}=g(\omega_n).
\end{equation}
In this case, the delay map is written as
\begin{equation}
	\mathbf{\Phi}_{f,g,\psi}(x,\omega) = \langle y_n, y_{n+\tau}, y_{n+2\tau},\ldots, y_{n+(\kappa-1)\tau} \rangle,
\end{equation}
and the theorem can be expressed as follows.
\begin{theorem}[Bundle Delay Embedding Theorem~\cite{stark97a}] \label{th:bundle-embedding}
	Let $\mathcal{M}$ and $\mathcal{N}$ be compact manifolds of dimension $d \geq 1$ and $e$ respectively. Suppose that $\kappa \geq 2(d+e)+1$ and the periodic orbits of period $\leq d$ of $g \in \mathcal{D}^r(\mathcal{N})$ are isolated and have distinct eigenvalues. Then, for $r \geq 1$, there exists an open and dense set of $(f,\psi) \subset \mathcal{D}^r(\mathcal{M}\times \mathcal{N},\mathcal{M}) \times \mathcal{C}^r(\mathcal{M},\mathbb{R})$ for which the map $\boldsymbol{\Phi}_{f,g,\psi}$ is an embedding of $\mathcal{M}\times \mathcal{N}$ into $\mathbb{R}^\kappa$.
\end{theorem}

Finally, all theorems up until now have assumed a single read-out function for the system in question. Recently, \ifpaper \cite{sugihara12a} \else Sugihara and Deyle~\cite{sugihara12a} \fi showed that multivariate mappings also form an embedding, with minor changes to the technical assumptions underlying Takens' original theorem. That is, given $M \leq 2d+1$ different observation functions, the delay map can be written as
\begin{equation}
	\mathbf{\Phi}_{f,\langle \psi^i \rangle}(\boldsymbol{x}) = \langle \mathbf{\Phi}_{f,\psi^1}(\boldsymbol{x}), \mathbf{\Phi}_{f,\psi^2}(\boldsymbol{x}),\ldots, \mathbf{\Phi}_{f,\psi^M}(\boldsymbol{x}) \rangle,
\end{equation}
where each delay map $\mathbf{\Phi}_{f,\psi^i}$ is as per~\eqref{eq:delay-embedding-map} for individual embedding dimension $\kappa^i \leq \kappa$. The theorem can then be stated as follows.
\begin{theorem}[Delay Embedding Theorem for Multivariate Observation Functions~\cite{deyle11a}] \label{th:delay-embedding-multi}
	Let $\mathcal{M}$ be a compact manifold of dimension $d \geq 1$. Consider a diffeomorphism $f\in\mathcal{D}^r(\mathcal{M},\mathcal{M})$ and a set of at most $2d+1$ observation functions $\langle \psi^i \rangle$ where each $\psi^i \in C^r(\mathcal{M},\mathbb{R})$ and $r \geq 2$. If $\sum_i \kappa^i \geq 2d + 1$, then, for generic $(f,\langle\psi^i\rangle)$, the map $\boldsymbol{\Phi}_{f,\langle\psi^i\rangle}$ is an embedding.
\end{theorem}

\ifpaper
	\bibliographystyle{apacite}
\else
	\bibliographystyle{ieeetr}
\fi
\bibliography{./structure-learning}

\end{document}